\def\year{2021}\relax
\documentclass[letterpaper]{article} % DO NOT CHANGE THIS
\usepackage{aaai21}  % DO NOT CHANGE THIS
\usepackage{times}  % DO NOT CHANGE THIS
\usepackage{helvet} % DO NOT CHANGE THIS
\usepackage{courier}  % DO NOT CHANGE THIS
\usepackage[hyphens]{url}  % DO NOT CHANGE THIS
\usepackage{graphicx} % DO NOT CHANGE THIS
\usepackage{natbib}  % DO NOT CHANGE THIS AND DO NOT ADD ANY OPTIONS TO IT
\usepackage{caption} % DO NOT CHANGE THIS AND DO NOT ADD ANY OPTIONS TO IT
\usepackage[math,gls]{alp}
\usepackage{algorithm} 
\usepackage{algpseudocode} 
\usepackage{subcaption}
\usepackage{amsthm}
\usepackage{hyperref}

\newtheorem{theorem}{Theorem}

\nocopyright
\title{GAN Ensemble for Anomaly Detection}

\newcommand*\samethanks[1][\value{footnote}]{\footnotemark[#1]}
\author{Xu Han\thanks{The first two authors contribute equally. This work has been accepted by AAAI 2021.},
    Xiaohui Chen\samethanks,
    Li-Ping Liu\\
}
\affiliations{
    Department of Computer Science, Tufts University\\
    \{Xu.Han, Xiaohui.Chen, Liping.Liu\}@tufts.edu
}

\begin{document}
\maketitle
\begin{abstract}
\begin{quote}
When formulated as an unsupervised learning problem, anomaly detection often requires a model to learn the distribution of normal data. Previous works apply Generative Adversarial Networks (GANs) to anomaly detection tasks and show good performances from these models. Motivated by the observation that GAN ensembles often outperform single GANs in generation tasks, we propose to construct GAN ensembles for anomaly detection. In the proposed method, a group of generators and a group of discriminators are trained together, so every generator gets feedback from multiple discriminators, and vice versa. Compared to a single GAN, a GAN ensemble can better model the distribution of normal data and thus better detect anomalies. Our theoretical analysis of GANs and GAN ensembles explains the role of a GAN discriminator in anomaly detection. In the empirical study, we evaluate ensembles constructed from four types of base models, and the results show that these ensembles clearly outperform single models in a series of tasks of anomaly detection. The code is available at \url{https://github.com/tufts-ml/GAN-Ensemble-for-Anomaly-Detection} ~.

\end{quote}
\end{abstract}
\section{Introduction}
Anomaly detection is an important problem in machine learning and has a wide range of applications such as fraud detection \citep{abdallah2016fraud, kou2004survey}, intrusion detection \citep{sabahi2008intrusion, garcia2009anomaly}, and event detection \citep{atefeh2015survey}.  In most anomaly detection problems, only samples of normal data are given, and the task is to detect \textit{anomalies} that deviates from the normal data. A large class of anomaly detection algorithms \citep{hodge2004survey,gupta2013outlier,chalapathy2019deep} directly or indirectly model the data distribution and then report samples atypical in the distribution as anomalies.

Generative Adversarial Networks (GAN) \citep{goodfellow2014generative}, which are flexible models for learning data distributions, provide a new set of tools for anomaly detection. A GAN consists of a generator network and a discriminator network, which learn a data distribution through a two-player game. If the training is ideal, the distribution defined by the generator should be the same as the data distribution. Several studies apply GANs to anomaly detection, such as AnoGAN~\citep{schlegl2017unsupervised}, f-AnoGAN~\citep{f-Anogan}, EGBAD~\citep{EGBAD}, GANomaly~\citep{akcay2018ganomaly}, and Skip-GANomaly~\citep{akccay2019skip}. All these models use an encoder-decoder as the generator since the generator of a vanilla GAN cannot check a sample.  A synthetic sample from the decoder is either a reconstruction of a real sample or a new sample. The discriminator needs to distinguish synthetic samples from the decoder and samples from the real data. The anomaly score is usually computed by checking the reconstruction of a sample and the internal representation of a sample in the discriminator. These models perform well in a series of detection tasks. 

There still lacks a thorough understanding of the role of adversarial training in anomaly detection. Theoretical analysis \citep{goodfellow2014generative,arora2017generalization} indicates that the discriminator should do no better than random guessing at the end of an ideal training procedure. However, the internal representation of a discriminator is very effective in differentiating normal samples and anomalies in practice. Then there is a gap between theory and practice: how does the discriminator characterize normal samples' distribution?

Training a GAN is challenging because it needs to optimize multiple deep neural networks in a min-max problem. The optimization often has stability issues. If neural networks are not carefully regularized, there are also problems of mode collapse. Recent works show multiple generators or/and discriminators help to overcome those problems. Several studies \citep{durugkar2016generative,neyshabur2017stabilizing} use multiple discriminators to provide stable gradients to the generator, making the training process more smooth. Multiple generators, which essentially defines a distribution mixture, can capture multiple data modes \citep{hoang2018mgan}. \citet{arora2017generalization} analyzes the equilibrium of GAN models and show that a mixture of both generators and discriminators guarantees an approximate equilibrium in a GAN's min-max game. These ensemble methods have improved the performance in various generation tasks.   

In this work, we propose to use GAN ensembles for anomaly detection. A GAN ensemble consists multiple encoder-decoders and discriminators, which are randomly paired and trained via adversarial training. In this procedure, an encoder-decoder gets feedback from multiple discriminators while a discriminator sees  ``training samples'' from multiple generators. The anomaly score is the average of anomaly scores computed from all encoder-decoder-discriminator pairs. This ensemble method works for most existing GAN-based anomaly detection models. 

We further analyze GAN models in the context of anomaly detection. The theoretical analysis refines the function form of the optimal discriminator in a WGAN~\citep{arjovsky2017wasserstein}. The analysis further explains why the discriminator helps to identify anomalies and how GAN ensemble improves the performance.  

In the empirical study, we test the proposed ensemble method on both synthetic and real datasets. The results indicate that ensembles significantly improve the performance over single detection models. The empirical analysis of vector representations verifies our theoretical analysis.

\section{Background}

We first formally define the anomaly detection problem. Suppose we have a training set $\bX$ of $N$ normal data samples,  $\bX = \{\bx_n \in \bbR^d : n=1, \ldots, N\}$ from some unknown distribution $\calD$. Then we have a test sample $\bx' \in \bbR^d$, which may or may not be from the distribution $\calD$. The problem of anomaly detection is to train a model from $\bX$ such that the model can classify $\bx'$ as \textit{normal} if $\bx'$ is from the distribution $\calD$ or \textit{abnormal} if $\bx'$ is from a different distribution. Often the time the model computes an anomaly score $y' \in \bbR$ for $\bx'$ and decide the label of $\bx'$ by thresholding $y'$. 

Anomaly detection essentially depends on the characterization of the distribution of normal samples. Adversarial training, which is designed to learn data distributions, suits the task well. Models based on adversarial training usually have an encoder-decoder as the generator and a classifier as the discriminator. Previous models such as f-AnoGAN, EGBAD, GANomaly, and Skip-GANomaly all share this architecture. It is important to have an encoder-decoder as the generator because a low-dimensional encoding of a new sample is often needed. Below is a review of these models.

The generator consists of an encoder $G_e(\cdot; \phi):\bbR^{d} \rightarrow \bbR^{d'}$, which is parameterized by $\phi$,  and a decoder $G_d(\cdot; \psi):\bbR^{d'} \rightarrow \bbR^{d}$, which is parameterized by $\psi$. The encoder maps a sample $\bx$ to an \textit{encoding vector} $\bz$ while the decoder computes a reconstruction $\tilde{\bx}$ of the sample from $\bz$. 
\begin{align}
\bz =  G_e(\bx; \phi), ~~  \tilde{\bx} =  G_d(\bz; \psi).
\end{align}
Skip-GANomaly uses U-Net with skip connections as the encoder-decoder, which does not compute $\bz$ explicitly.

The discriminator $D(\cdot; \gamma)$, which is parameterized by $\gamma$, takes a sample and predict the probability of the sample being from the data $\bX$ instead of the generator. In the context of WGAN, the critic function play a similar role as a discriminator, so we also denote it as $D(\cdot; \gamma)$ and call it ``discriminator'' for easy discussion. The discriminator $D(\cdot; \gamma)$ should give higher values to normal data samples than reconstructions. The discriminator from vanilla GAN and a WGAN has the form of $u =  D(\bx; \gamma)$.  In a EGBAD model, which is based on a BiGAN, the discriminator takes both a sample and its encoding as the input, so it has the from of $u = D((\bx, G_e(\bz)); \gamma)$.

Since a model consists of an encoder-decoder and a discriminator, the training often considers losses inherited from both models. The \textit{adversarial loss} is from GAN training. The losses are defined as follows when the discriminator works in three GAN types (vanilla GAN, WGAN, and BiGAN). 
\begin{align}
L_{a{\text -}g}(\bx)  \hspace{0.5em} &= \hspace{0.2em} \log D(\bx) + \log \big(1 -  D(G_d(G_e(\bx)))\big) \label{eq:gan-loss}\\
L_{a{\text -}wg}(\bx) \hspace{0em  } &= \hspace{0.2em} D(\bx) - D\big(G_d(\tilde{\bz})\big)  \label{eq:wgan-loss}\\
L_{a{\text -}bg}(\bx) \hspace{0.2em} &= \hspace{0.2em} \log D(\bx, G_e(\bx)) + \log \big(1 -  D(G_d(\tilde{\bz}), \tilde{\bz})\big) \label{eq:bigan-loss}
\end{align}
In $L_{a{\text -}g}(\bx)$, the sample from the generator is the reconstruction of the original sample. The WGAN objective will be used by f-AnoGAN, which does not train the generator $G_e(\cdot; \phi)$ in the objective, so $L_{a{\text -}wg}(\bx)$ has no relation with $\phi$. Here we assume that $\tilde{\bz}$ is sampled from a prior distribution $p(\bz)$, e.g. multivariate Gausian distribution, so $L_{a{\text -}wg}(\bx)$ and $L_{a{\text -}bg}(\bx)$ are considered as stochastic functions of $\bx$.  

The second type of loss is the \textit{reconstruction loss}, which is often used in training encoder-decoders. The loss is the difference between a reconstruction and the original sample. The difference is measured by $\ell$-norm with either $\ell=1$ or $\ell=2$.
\begin{align}
L_{r}(\bx) = \|\bx - G_{d}\big(G_e(\bx)\big)\|_{\ell}^{\ell}
\end{align}

Previous research also indicates that the \textit{hidden vector} $\bh$ of a sample in the last hidden layer of the discriminator $D(\cdot; \gamma)$ is often useful to differentiate normal and abnormal samples. 
Denote $\bh = f_{D}(\bx; \gamma)$ as the hidden vector of $\bx$ in $D(\bx; \gamma)$, then the \textit{discriminative loss} based on $\bh$ is  
\begin{align}
L_{d}(\bx) = \| f_{D}(\bx) -  f_{D}\big(G_{d}(G_e(\bx))\big)\|_{\ell}^{\ell} \label{eq:loss-d}
\end{align}

The GANomaly model also considers the difference between the encoding vectors of a normal sample $\bx$ and its reconstruction $\tilde{\bx}$. Particularly, it uses a separate encoder $G_e(\cdot; \tilde{\phi})$ to encode the recovery $\tilde{\bx}$. Then the \textit{encoding loss} is  
\begin{align}
L_{e}(\bx) =  \| G_{e}(\bx; \phi) -  G_e\big(G_{d}(G_e(\bx; \phi)); \tilde{\phi}\big)\|_{\ell}^{\ell}
\end{align}
We explicitly indicates that the parameters $\phi$ and $\tilde{\phi}$ of the two encoders are different. 

To train their discriminators, these GAN models need to maximize the adversarial loss. 
\begin{align}
\max_{\gamma} \sum_{i=1}^N L_{a}(\bx_i; \phi, \psi, \gamma) 
\end{align}
Here $L_a$ can be any of the three losses in  \eqref{eq:gan-loss}, \eqref{eq:wgan-loss}, and \eqref{eq:bigan-loss}. Here we explicitly show all trainable parameters in the loss function. %F-AnoGAN uses $L_{a{\text -}wg}$ to train a WGAN while EGBAD uses $L_{a{\text -}bg}(\cdot)$ to train a BiGAN. GANomaly and Skip-GANomaly uses $L_{a{\text -}g}$ to train a standard GAN model. The discriminator of Skip-GANomaly also receives gradient from the minimization of the loss $L_{d}$. 

To train their generators, different models minimize some of the four losses mentioned above. We write their objective in a general form.  
\begin{multline}
\min_{\phi, \psi, \tilde{\phi}} \sum_{i=1}^N \alpha_1 L_{a}(\bx_i; \phi, \psi, \gamma) +  \alpha_2 L_{r}(\bx_i; \phi, \psi) \\ + \alpha_3 L_{e}(\bx_i; \phi, \psi, \tilde{\phi}) +  \alpha_4 L_{d}(\bx_i; \phi, \psi, \gamma)
\end{multline}
F-AnoGAN \citep{f-Anogan} trains the decoder and encoder separately. It first trains the decoder $G_d(\cdot; \psi)$ and the discriminator $D(\cdot; \gamma)$ by setting $\alpha_2 = \alpha_3 = \alpha_4 = 0$. Then it trains the encoder with the decoder and the discriminator fixed and $\alpha_1=\alpha_3=0$. 

After training a model, we need to compute an anomaly score $A(\bx')$ for a test instance $\bx'$. The anomaly score is usually a weighted sum of the reconstruction loss and the discriminative loss.    
\begin{align}
A(\bx') = L_{r}(\bx') + \beta L_{d}(\bx')
\label{eq:anomaly-score}
\end{align}
The relative weight $\beta$ is usually empirically selected.  The exception is GANomaly, which uses the encoding loss $A(\bx')= L_{e}(\bx')$ as the anomaly score. The higher an anomaly score is, the more likely $\bx'$ is an anomaly.

\section{Proposed Approach}

% Note: this sentence belongs to the introduction section  
%As it is pointed out by \cite{theis2015note}, GAN potentially do not describe the whole data distribution, which in this scenario misclassifying normal sample into abnormal sample is likely to happen. To address this,  we investigate the usage of multiple GANs and evaluate the performance gain. 

In this work, we propose an ensemble learning framework for anomaly detection based on GAN models. We will have multiple generators and discriminators, which are parameterized differently.  We define $I$ generators, $\mathcal{G}=\{\left(G_e(\cdot; \phi_i), G_d(\cdot; \psi_i)\right): i = 1, \ldots, I\}$, and $J$ discriminators $\mathcal{D}=\{D(\cdot; \gamma_j): j = 1, \ldots, J\}$. A single generator or discriminator is the same as that of a base model. GANomaly has two encoders, so its generators are $\mathcal{G}=\left\{\left(G_e(\cdot; \phi_i), G_d(\cdot; \psi_i); G_e(\cdot; \tilde{\phi}_i)\right): i = 1, \ldots, I \right\}$ in this framework. We will omit $G_e(\cdot; \tilde{\phi}_i)$ in the following discussion, but the optimization with GANomaly should be clear from the context.  

In the adversarial training process,  we pair up every generator with every discriminator. Then a generator is critiqued by every discriminator, and a discriminator receives synthetic samples from every generator. In our method, we do not use weights for generators or discriminators as \citep{arora2017generalization} since we want to sufficiently train every generator and every discriminator. 

With multiple pairs of generators and discriminators, the adversarial loss and the discriminative loss are both computed from all generator-discriminator pairs. Denote the losses between each generator-discriminator pair $(i,j)$ by
\begin{align}
L^{ij}_{a} = L_{a}(\bx; \phi_i, \psi_i, \gamma_j), ~~ L^{ij}_{d} = L_{d}(\bx; \phi_i, \psi_i, \gamma_j).
\end{align}
Here $L_{a}(\bx; \phi_i, \psi_i, \gamma_j)$ can be any of the three adversarial loss functions defined in \eqref{eq:gan-loss}, \eqref{eq:wgan-loss}, and \eqref{eq:bigan-loss}. We explicitly write out parameters of the pair of generator and discriminator to show the actual calculation.  

Similarly denote the recovery loss and the encoding loss of a single generator $i$ by
\begin{align}
L_{r}^{i} =  L_{r}(\bx; \phi_i, \psi_i), ~~L_{e}^{i} = L_{e}(\bx; \phi_i, \psi_i). 
\end{align}

Then we maximize the sum of adversarial losses to train discriminators while minimize the sum of all losses to train generators. The training objectives are as follows. 
\begin{align}
\max_{(\gamma_j)_{j = 1}^J} & \sum_{i=1}^I \sum_{j=1}^{J} L_{a}^{ij} \label{eq:dis-obj}\\
\min_{(\phi_i, \psi_i)_{i=1}^{I}} & \sum_{i=1}^{I}\sum_{j=1}^{J} \alpha_1 L^{(ij)}_{a} + \alpha_2 L^{(i)}_{r} +  \alpha_3  L^{(ij)}_{d}  + \alpha_4 L^{(i)}_{e}  \label{eq:enc-obj}
\end{align}
This ensemble method works with all base models reviewed in the last section. The training of the ensemble is shown in Figure \ref{fig:framework}. 

\subsubsection{Batch training} In one training iteration we update only one generator-discriminator pair instead of all generators and discriminators. Particularly, we randomly pick a generator and a discriminator and compute the loss with a random batch of training data. With the spirit of stochastic optimization, we still optimize the objectives in \eqref{eq:dis-obj} and \eqref{eq:enc-obj}. The training algorithm is shown in Algorithm \ref{alg:main}. Note that the training does not take as much as $(IJ)$ times of the training time of a base model -- it is much faster than that. This is because a generator is updated once in $I$ iterations on average, and it is similar for a discriminator. In actual implementations, small $I$ and $J$ values (e.g. $I=J=3$) often provide significant performance improvement. The training of the model is shown in Algorithm \ref{alg:main}. 

\begin{figure}[t]
\begin{center}
\includegraphics[width=0.38\textwidth]{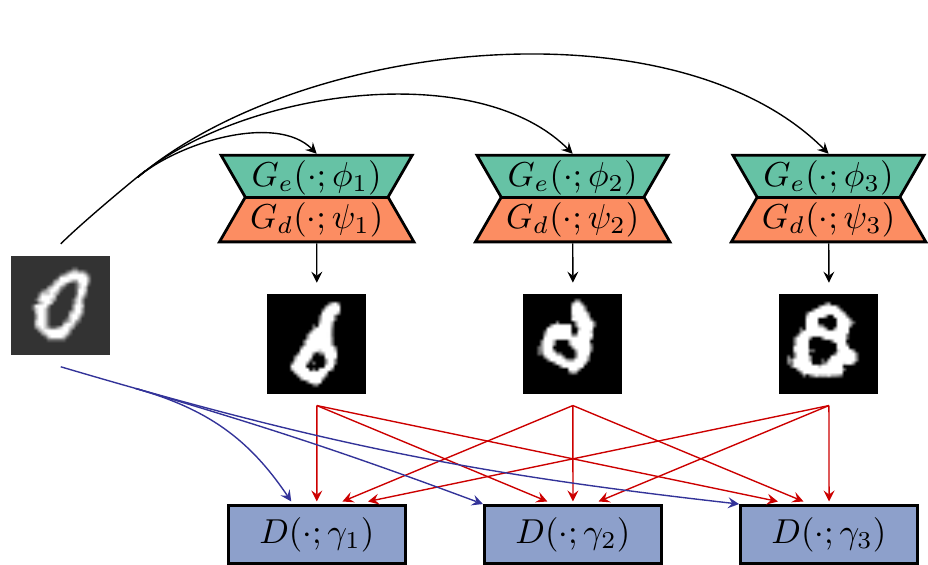}
\end{center}
\caption{GAN ensemble with multiple generators and discriminators. The ``0'' image on the left is an anomaly, and the three images from the three encoder-decoders are reconstructions of the ``0'' image.}
\label{fig:framework}
\end{figure}

\subsubsection{Anomaly Score}
The anomaly score of the ensemble $\calA(\bx')$ for a new instance $\bx'$ is the average of anomaly scores from multiple generators and discriminators. 
\begin{equation}
 \calA(\bx')=\frac{1}{IJ}\sum_{i=1}^I\sum_{j=1}^J A(\bx'; \phi_i, \psi_i, \gamma_j). 
 \label{eq:anomaly-score-en}
\end{equation}
The average of anomaly scores helps to smooth out spurious scores if a model is not well trained at a specific test instance.

\begin{algorithm}[t]
	\caption{GAN ensemble for anomaly detection}
    \label{alg:main}
	\textbf{Input:} Training set $\bX=\{\bx_i\}_{i=1}^{N}$ \\
    \textbf{Output:} Trained generators $\{(G_e(\cdot;\phi_i), G_d(\cdot;\psi_i)\}_{i=1}^{I}$ and discriminators $\{D(\cdot; \gamma_j)\}_{j=1}^{J}$
	\begin{algorithmic}[1]
	\State Initialize parameters for$(\phi_i, \psi_i)_{i=1}^{I}$ an $(\gamma_j)_{j=1}^{J}$ 
	\State $t\leftarrow 0$
	\While {the objective not converge and  $t < \mbox{max\_iter}$}
	    \State Sample $i$ from \{1, \ldots, I\} and $j$ from $\{1, \ldots, J\}$
		\State Sample a minibatch $\bX^t$ from $\bX$
		\State Compute the adversearial loss $L_a^{(ij)}$
        \State Update $D(\cdot; \gamma_j)$: $\gamma_j \leftarrow \gamma_j + \nabla_{\gamma_j} L_a^{(ij)}$
		\State $\small \calL^{(ij)} = \alpha_1 L^{(ij)}_{a} + \alpha_2 L^{(i)}_{r} +  \alpha_3  L^{(ij)}_{d}  + \alpha_4 L^{(i)}_{e}$
        \State Update $G_e(\cdot;\phi_i)$: $\phi_i \leftarrow \phi_i - \nabla_{\phi_i} \calL^{(ij)}$
        \State Update $G_d(\cdot;\psi_i)$: $\psi_i \leftarrow \psi_i - \nabla_{\psi_i} \calL^{(ij)}$
		\State $t\leftarrow t+1$
	\EndWhile
	\end{algorithmic}
\end{algorithm}
\section{Analysis}
\label{sect:analysis}
In this section we analyze GANs and GAN ensembles in the context of anomaly detection. The discriminator plays an important the role in anomaly detection. By the GAN training objective, a discriminator needs to give small values to synthetic samples that are far from real data. Some previous work analyzes the GAN equilibrium using a distribution of real data and then have the conclusion that the discriminator does no better than random guess at the equilibrium. In practice, a GAN essentially approximates discrete data distribution with a continuous generative distribution. The discriminator at convergence actually give large values to training samples and small values to samples that are far away.  

We first analyze the Wasserstein GAN \citep{arjovsky2017wasserstein, zhou2019lipschitz}, and the results directly apply to f-AnoGAN. We consider the case that the critic function $D(\cdot)$ (also called as discriminator in the discussion above) is only restricted by the 1-Lipschitz condition.  Given a generator $G_d(\cdot)$, the maximization objective of $D(\cdot)$ is \citep{arjovsky2017wasserstein}   
\begin{align}
\max_{\|D(\cdot)\|_{L\le 1}} \E{\bx \sim \mathrm{unif}(\mathrm{\bX})}{D(\bx)} - \E{\bx \sim G_d(\cdot)}{D(\bx)} \label{eq:opt-d}
\end{align}
Here we directly consider training samples in a uniform distribution because the actual training use these samples instead of an imaginary real distribution. 

To define the Lipschitz continuity, we assume there is a norm $\| \cdot \|$ defined for the sample space $\bbR^d$. $D(\cdot)$ is restricted to be 1-Lipschitz with respect to this norm. The following theorem shows that the function $D(\cdot)$ is decided by it function values $\{D(\bx_i): i = 1,\ldots, N\}$ on real samples $\bX$ and the support $\calS$ of the generator $G_d(\cdot)$. 
\begin{theorem} Assume the generator $G_d(\cdot)$ defines a continuous distribution that has positive density on its support $\calS$. Suppose the optimizer $D^*(\cdot)$ of \eqref{eq:opt-d} takes values $(D^*(\bx_i): i = 1, \ldots, N)$ on training samples $\bX$, then it must have the following form,  
\begin{align}
D^*(\bx) = \max_{i} D^*(\bx_i) - \| \bx - \bx_i \|, ~~~~~~ \forall \bx \in \calS. \label{eq:d-form}
\end{align}
Furthermore, the value of $D^*(\bx)$ for $\bx \notin (\calS \cup \bX)$ does not affect the objective \eqref{eq:opt-d}.
\end{theorem}
\begin{proof}
We first show that $D^*(\bx)$ achieves the smallest value on any sample in $\calS - \bX$ under the 1-Lipschitz constraint. Let $\bx' \in \calS - \bX$ and $i' =  \argmax_{i} D^*(\bx_i) - \| \bx' - \bx_i \|$, then $D^*(\bx') = D^*(\bx_{i'}) - \| \bx' - \bx_{i'} \|$. $D^*(\bx')$ cannot take any value smaller than that; otherwise it violates the 1-Lipschitz constraint. 

Then we show that $D^*(\cdot)$ itself is 1-Lipschitz. Suppose we have another sample $\bx'' \in \calS - \bX$ and $i'' = \argmax_{i} D^*(\bx_i) - \| \bx'' - \bx_i \|$, then $D^*(\bx'') = D^*(\bx_{i''}) - \| \bx'' - \bx_{i''} \| \ge f(\bx_{i'}) - \|\bx_{i'} - \bx''\|$. We have $D^*(\bx') - D^*(\bx'') \le - \|\bx_{i'} - \bx'\| + \|\bx_{i'} - \bx''\| \le \|\bx' - \bx''\|$ by the triangle inequality. We also have  $D^*(\bx') - D^*(\bx'') \ge  - \|\bx' - \bx''\|$ by switching  $\bx'$ and $\bx''$ and using the same argument. Therefore,   $D^*(\bx)$ is 1-Lipschitz.  

For $\bx \notin (\calS \cup \bX)$, $D^*(\bx)$ is not considered in the calculation in the objective, so it does not affect the objective.  
\end{proof}

\begin{figure}
\centering
\includegraphics[width=0.4\textwidth]{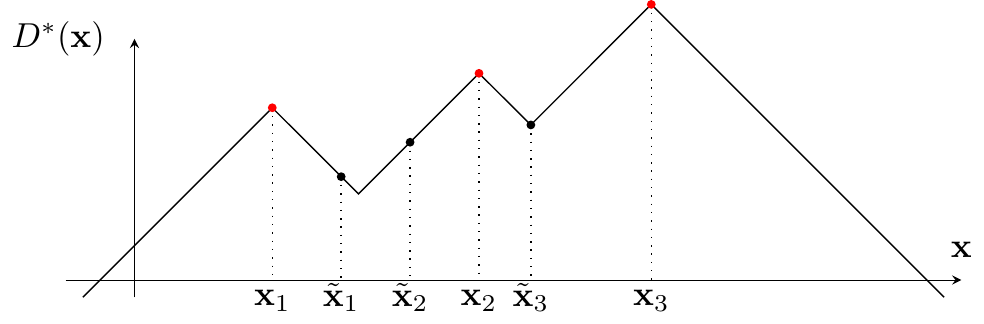}
\caption{An illustration of $D^*(\bx)$ in Theorem 1.}
\label{fig:discriminator}
\end{figure}
An illustration of the function $D^*(\bx)$ is shown in Figure \ref{fig:discriminator}. From the theorem, we form two views about why a discriminator helps to identify anomalies.

\subsubsection{Anomalies at the discriminator}
The theorem indicates that the discriminator should give small values to samples that are generated by the generator, particularly those far from training samples. In actual training, the generator often cannot define a distribution that matches the training data well. The generated samples that are very different from training samples become ''positive samples'' (e.g. shirts with asymmetrical sleeves) of anomalies for the discriminator, then the discriminator is trained to give small values to these samples and also anomalies. A similar principle should exist in other types of GANs. This explains why we can include discriminator outputs in anomaly scores \citep{EGBAD}. 

The hidden vector in the last layer of the discriminator is often very different for an anomaly and its reconstruction. The difference is also used to compute an anomaly score in \eqref{eq:loss-d} in f-AnoGAN, EGBAD, and Skip-GANomaly. We also have an explanation for this calculation. The reconstruction of an anomaly is often like a normal example because the unique property of an anomaly often cannot go through the encoder. In Figure \ref{fig:framework}, the sample ``0'' is the abnormal class, and its reconstructions are more like normal samples. Then the discriminator often gives very different values to an anomaly and its reconstruction. These values are just a linear transformations of hidden vectors in the last layer, so the hidden vector of an anomaly is different from that of its reconstruction. A normal sample and its reconstruction are often very similar, and their hidden vectors are also similar, so a normal sample has small discriminative loss. 

\begin{table*}[t]
\centering
\caption{AUROC results on the MNIST dataset. Ensemble methods generally outperform base models.}
\begin{tabular}{l|cccccccccc|c}
\hline
& 0 & 1 & 2 & 3 & 4 & 5 & 6 & 7 & 8 & 9 & average \\\hline
f-AnoGAN & 0.895 &0.699 & 0.863 & 0.785 & 0.781 & {0.761} & 0.896 & 0.702 & \underline{\textbf{0.889}} & 0.630 & 0.799 \\

EGBAD & 0.782 & \underline{0.298} & 0.663 & 0.511 & \underline{0.458} & 0.429 & 0.571 & 0.385 & 0.521 & 0.342 &0.496\\

GANomaly   & 0.881 & \underline{0.661} & \underline{\textbf{0.951}} & 0.796 & 0.809 & 0.868 & 0.859 & 0.671 & 0.653 & 0.533 & 0.673\\\hline
$\text{f-AnoGAN}^\textbf{en}$ & \underline{\textbf{0.961}} & \underline{\textbf{0.943}} & \underline{0.914} & \underline{\textbf{0.913}} & \underline{0.817} & \underline{0.767} & \underline{\textbf{0.957}} & \underline{0.782} & 0.830 & \underline{0.681} & \underline{\textbf{0.857}}\\

$\text{EGBAD}^\textbf{en}$ &  \underline{0.804}&  0.202& \underline{0.671} & \underline{0.577} & 0.438  &\underline{0.480}  &\underline{0.595}  & \underline{0.425} & \underline{0.595}  &\underline{0.458} & \underline{0.525}\\

$\text{GANomaly}^\textbf{en}$   & \underline{0.901} & 0.598 & 0.931 & \underline{0.883} & \underline{\textbf{0.838}} & \underline{\textbf{0.875}} & \underline{0.892} & \underline{\textbf{0.801}} & \underline{0.847}& \underline{\textbf{0.733}} & \underline{0.830}\\\hline
\end{tabular}
\label{table:mnist}
\end{table*}

\begin{table*}[t]
\caption{AUROC results on the CIFAR-10 dataset. Ensemble methods generally outperform base models.}
\centering
\begin{tabular}{l|cccccccccc|c}
\hline
& bird & car & cat & deer & dog & frog & horse & plane & ship & truck & average \\\hline
f-AnoGAN & 0.427 &0.778 & 0.541 & 0.442 & 0.601 & 0.582 & 0.628 & 0.688 & 0.637 & 0.781 & 0.611\\
EGBAD & 0.383 & 0.514 & 0.448 & 0.374 & \underline{0.481} & 0.353 & \underline{0.526} & \underline{0.577} & 0.413 & 0.555 & 0.462\\
GANomaly   & 0.510 & 0.631 & 0.587 & 0.593 & 0.628 & 0.683 & 0.605 & 0.633 & 0.710 & 0.617 & 0.620\\
Skip-GANomaly & 0.448& \underline{\textbf{0.953}}& 0.607& 0.602& 0.615& \underline{\textbf{0.931}}& 0.788& \underline{0.797}& 0.659& 0.907 &0.731\\\hline
$\text{f-AnoGAN}^\textbf{en}$ & \underline{0.531} & \underline{0.804} & \underline{0.581} & \underline{0.584} & \underline{0.616} & \underline{0.642} & \underline{0.653} & \underline{0.704} & \underline{0.830} & \underline{0.907} & \underline{0.685}\\
$\text{EGBAD}^\textbf{en}$& \underline{0.573} & \underline{0.620}  & \underline{0.451}  &\underline{0.563}  &0.388  &\underline{0.554}  & 0.429 & 0.522 & \underline{0.612} &\underline{0.668} &\underline{0.538}   \\
$\text{GANomaly}^\textbf{en}$  & \underline{0.533} & \underline{0.669} & \underline{0.599} & \underline{0.719} & \underline{0.667} & \underline{0.856} & \underline{0.614} & \underline{\textbf{0.948}} & \underline{0.854} & \underline{0.682} & \underline{0.714}\\
$\text{Skip-GANomaly}^\textbf{en}$ & \underline{\textbf{0.998}} & 0.917& \underline{\textbf{0.691}}& \underline{\textbf{0.766}}& \underline{\textbf{0.937}}& 0.764&\underline{\textbf{0.992}}& 0.703& \underline{\textbf{0.991}}& \underline{\textbf{0.917}}&\underline{\textbf{0.868}}\\\hline
\end{tabular}
\label{table:cifar}
\end{table*}

\subsubsection{Reconstruction guided by the discriminator} We have the form of ideal discriminator $D^*(\cdot)$ given the norm $\|\cdot\|$, but it is hard to find out the norm the discriminator uses. However, we can get some information about the norm by considering two similar samples, e.g.  $\bx_{i'}$ and  $\bx'$ when $\bx'$ is a good reconstruction of a normal sample $\bx_{i'}$. The training sample $\bx_{i'}$ is likely to be nearest to $\bx'$ among all training samples, and then $\| \bx' - \bx_{i'} \| = (D^*(\bx_{i'}) -  D^*(\bx'))$. The minimization of \eqref{eq:opt-d} with respect to the generator is essentially the minimization of $\| \bx' - \bx_{i'} \|$, which corresponds to the minimization of $\| \bx' - \bx_{i'} \|_2^2$ in an encoder-decoder. Therefore, a discriminator implicitly defines a norm over samples. The adversarial loss, which is like a reconstruction loss using this norm, guides the training of the encoder-decoder.  

\subsubsection{The benefit of GAN ensembles}
Previous research \citep{arora2017generalization} shows that multiple generators helps to capture data modes and provide synthetic samples with varieties. Multiple generators are also likely to have a larger joint support $\calS$ than a single one. Therefore, the generator ensemble trains a discriminator better than a single model. Our experiments later show that from the better training of the discriminator is very important.

\section{Experiments}

In this section, we evaluate the proposed method in several anomaly detection tasks. We also analyze the GAN ensemble with the experiment. 

We evaluate our method against baseline methods on four datasets: KDD99 \citep{Dua:2019} is a dataset for anomaly detection. OCT \citep{oct} has three classes with small number of samples, and these three classes are treated as abnormal classes. MNIST \citep{lecun-mnisthandwrittendigit-2010} and CIFAR-10 \citep{krizhevsky2009learning} are datasets for multiclass classification. We leave a class as the  abnormal class and use other classes as normal data. Among these four datasets, MNIST and CIFAR-10 contain low-resolution images, OCT has high-resolution images from clinical applications, and KDD99 is a non-image dataset.  

We consider four base models in our proposed framework: f-AnoGAN \citep{f-Anogan}, EGBAD \citep{EGBAD}, GANomaly \citep{akcay2018ganomaly} and Skip-GANomaly \cite{akccay2019skip}. Then we have four types of ensembles, and we compare these ensembles against base models. All the experiments are conducted with three $I=3$ generators and three $J=3$ discriminators. $I$ and $J$ are chosen for a good tradeoff between running time and performance. The appendix has more details about datasets and model training.

\subsection{Comparisons of detection performances}

\subsubsection{MNIST dataset}
MNIST has 10 classes of hand-written digits from 0 to 9. Every time we treat one class as anomalous and use other classes as normal data. We compare three ensemble methods and three base models. We use the AUROC score as the evaluation measure. The performance values of EGBAD and GANomaly are taken from the original papers.When an ensemble and its base model are compared, the better result is underlined. The best result across all models is shown in bold.

Table \ref{table:mnist} reports performances of different methods. 
Ensemble methods are indicated by the superscript ${\textbf{en}}$. In general, ensemble methods have superior performance over their corresponding base models on most of the 10 classes. For classes 1, 3, 7, and 9, the best ensemble method improves more than 0.1 in AUROC over baseline methods. On average, the ensembles of f-AnoGAN and GANomaly improves the performances of their respective base models by 0.07 and 0.23. We do not include Skip-GANomaly because this base model performs poorly on the MNIST dataset. 
% All models were trained only with normal data and tested with both normal and anomalous data. Table \ref{table:mnist} shows the clear superiority of the ensemble approach over the corresponding baseline models. In f-AnoGAN, EGBAD, GANomaly, applying ensemble achieves higher AUROC performance for 8 out of 10 digits chosen as anomalies. Especially, f-AnoGAN and GANomaly gain substantial improvement over most of the classes, while the EGBAD approach benefits less.

\subsubsection{CIFAR-10 dataset}
% In a similar vein to the MNIST dataset, we created 10 different dataset from the CIFAR-10 dataset. For each synthesized dataset, we have 45,000 normal training samples in total from 9 classes, and 15,000 test samples from 10 classes. The test set contains 9,000 normal samples and 6,000 anomalous samples.
% We employ the same strategy as MNIST to create 10 different datasets within the use of CIFAR-10 dataset \citep{krizhevsky2009learning}. Each synthesized dataset contains 45,000 normal training samples from 9 classes, and 15,000 test samples, of which has 9,000 normal samples and 6,000 anomalous samples. 
CIFAR-10 has 10 classes of images of size $32 \times 32$. Every time one class is treated  as anomalous, and other classes are used as normal data. All four base models and their ensembles are compared and evaluated by AUROC scores. The performance values of EGBAD, GANomaly, and Skip-GANomaly are taken from the original papers. %When an ensemble and its base model are compared, the better result is underlined. The best result across all models is shown in bold.

We report AUROC scores of all models in table \ref{table:cifar}. Ensembles constructed from f-AnoGAN and GANomaly outperform their corresponding base models on all 10 classes. The improvements from our ensemble method for EGBAD and Skip-GANomaly are also apparent in the results. The ensemble of Skip-GANomaly models, which performs the best on 7 out of 10 classes, is generally the best model for this dataset. 

\begin{table}[t]
\caption{AUROC results on the OCT dataset. Ensemble methods generally outperform base models.}
\label{table:oct}
\centering
\begin{tabular}
{l|ccc|c}
\hline
Method & CNV & DME & DRUSEN & overall\\\hline
f-AnoGAN &0.863 & 0.754 & 0.610 & 0.666\\
EGBAD &0.793 & 0.782 & 0.601 & 0.610\\
GANomaly &0.852 & 0.796 & 0.543 & 0.637\\
Skip-GANomaly &0.915 & \textbf{\underline{0.929}}& 0.609 & 0.766\\\hline
$\text{f-AnoGAN}^\textbf{en}$ &\underline{0.949} & \underline{0.866} & \underline{0.705} & \underline{0.800}\\
$\text{EGBAD}^\textbf{en}$ &\underline{0.889}&\underline{0.821}&\underline{0.611}&\underline{0.705}\\
$\text{GANomaly}^\textbf{en}$ &\underline{0.911} & \underline{0.845} & \underline{0.626} & \underline{0.741}\\
$\text{Skip-GANomaly}^\textbf{en}$ &\underline{\textbf{ 0.985}} & 0.889 & \underline{\textbf{0.952}} & \underline{\textbf{0.869}}\\
\hline
\end{tabular}
\end{table}

\begin{figure}[t]
\centering
\includegraphics[width=0.4\textwidth]{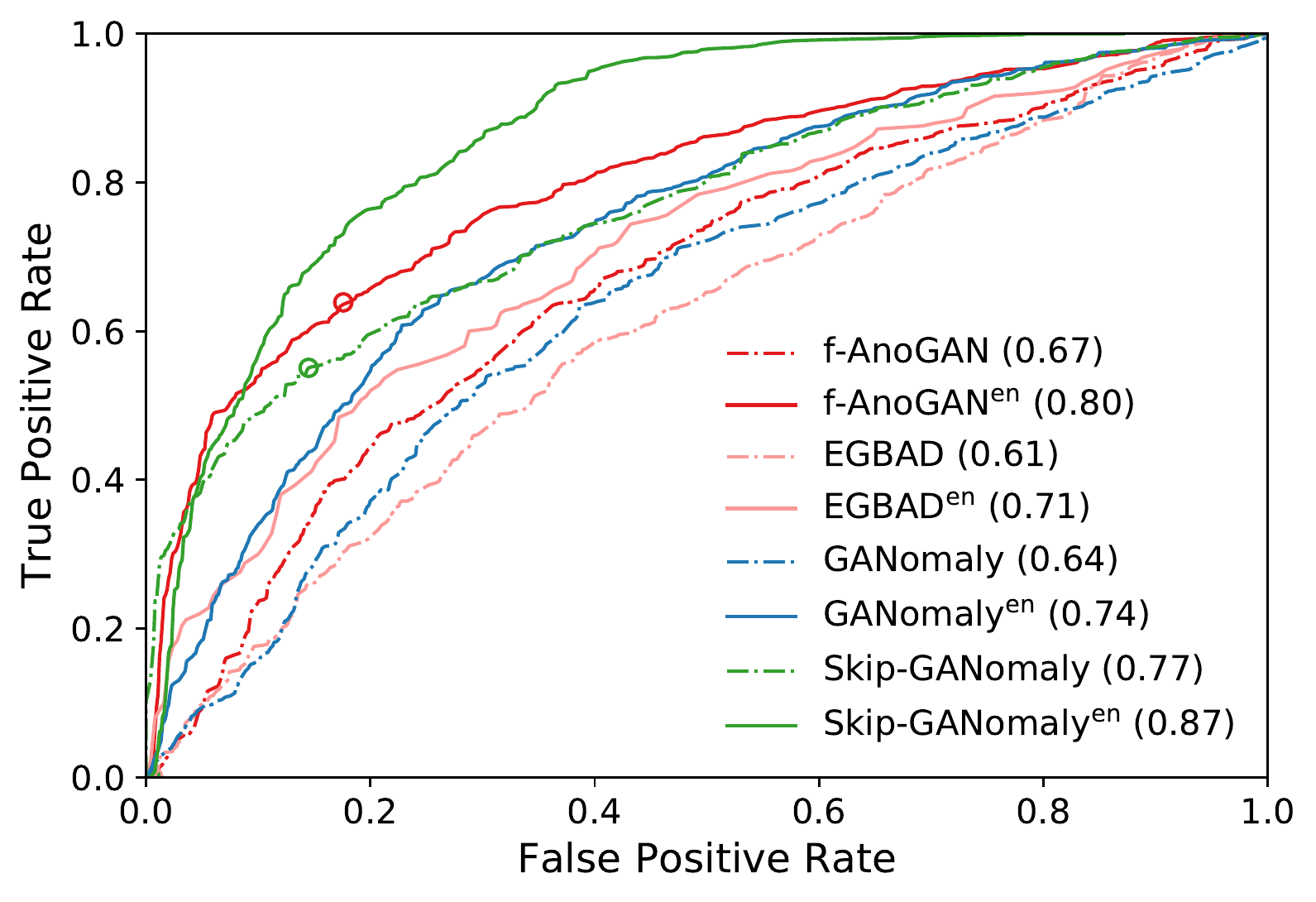}
\caption{ROC curves of different models on the OCT dataset with ``overall'' being the anomlous class.}
\label{fig:oct_model_comparison}
\end{figure}

\subsubsection{OCT dataset} 
OCT contains  high-resolution clinical images. It has three small classes (CNV, DME, DRUSEN) and a large class of images. We use the large class as normal data; then, we use three small classes separately as anomalous classes and use the three classes together as a single anomalous class (denoted as ``overall'').  

Table \ref{table:oct} shows AUROC performances of different methods over four types of anomaly categories (CNV, DME, DRUSEN, and overall). The ensemble method's benefit is more obvious in this dataset: except Skip-GANomaly on the DME class, all ensemble models show significant performance improvements on all anomaly classes. Particularly on the DRUSEN class, the ensemble of Skip-GANomaly has a much better performance than all base models. Figure \ref{fig:oct_model_comparison} shows ROC Curves for all single models as well as ensemble models. This result indicates that an ensemble model can make reliable decisions in medical applications.

\begin{figure*}[t]
\centering
\begin{tabular}{ccc|ccc}
% \multicolumn{4}{c}{\Large{latent in discriminator(s)}} \\
\multicolumn{3}{c|}{\large{Base model}} & \multicolumn{3}{c}{\large{Ensemble}} \\
\multicolumn{1}{c}{\large{generator}} & \multicolumn{2}{c|}{\large{discriminator}} &\multicolumn{1}{c}{\large{generators}} & \multicolumn{2}{c}{\large{discriminators}}\\
\includegraphics[width=0.14\textwidth]{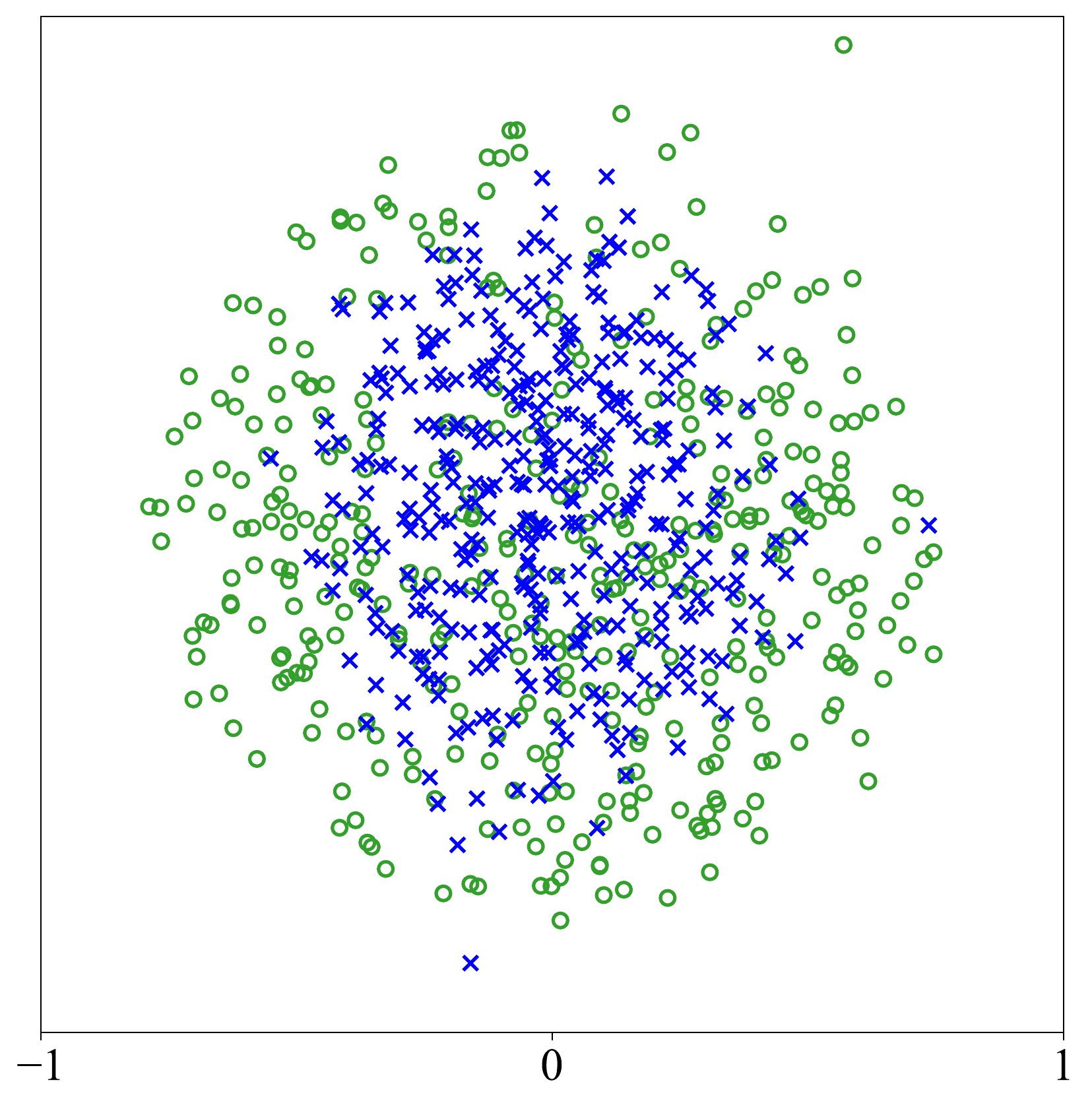}&
\includegraphics[width=0.14\textwidth]{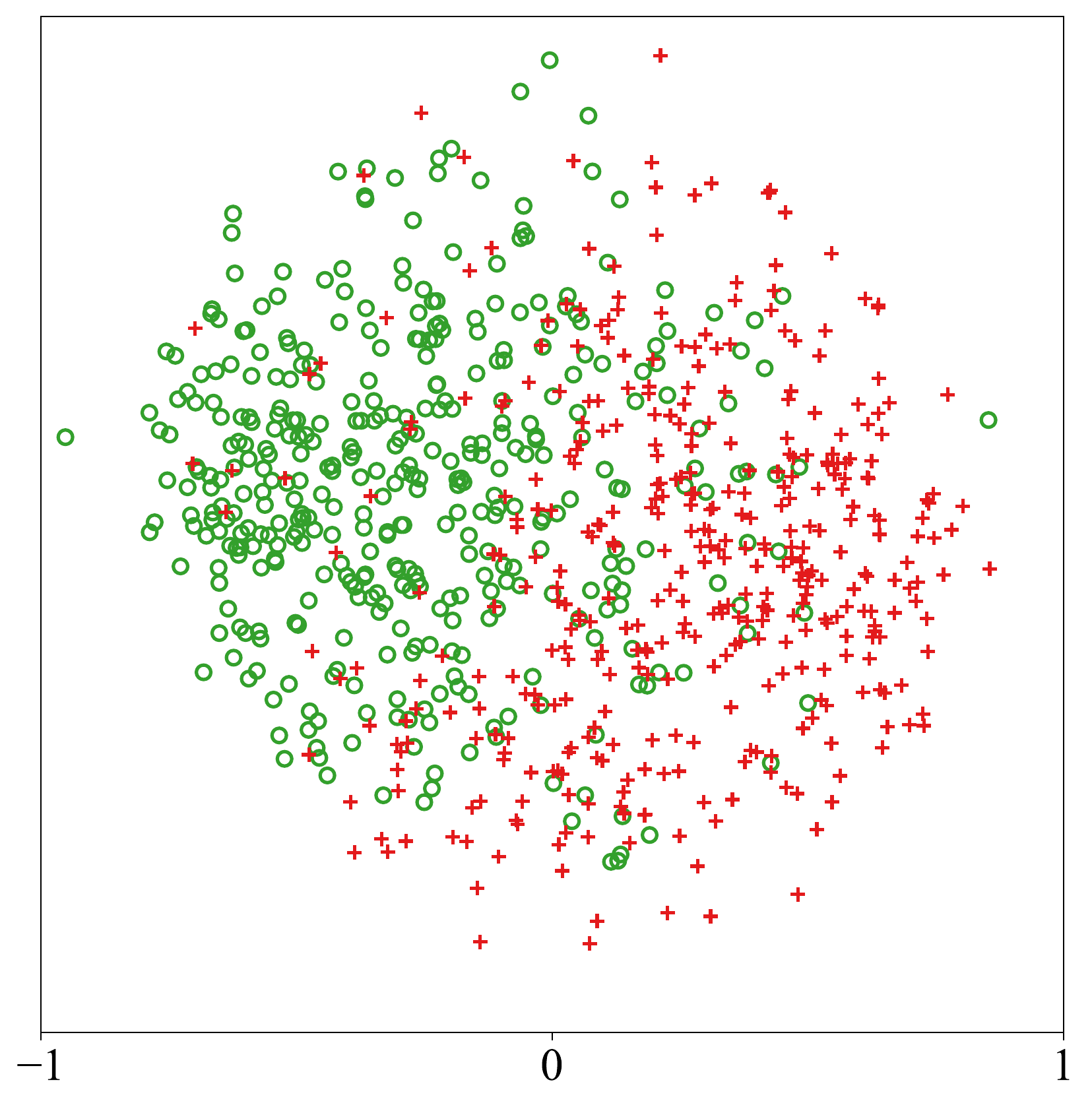} &
\includegraphics[width=0.14\textwidth]{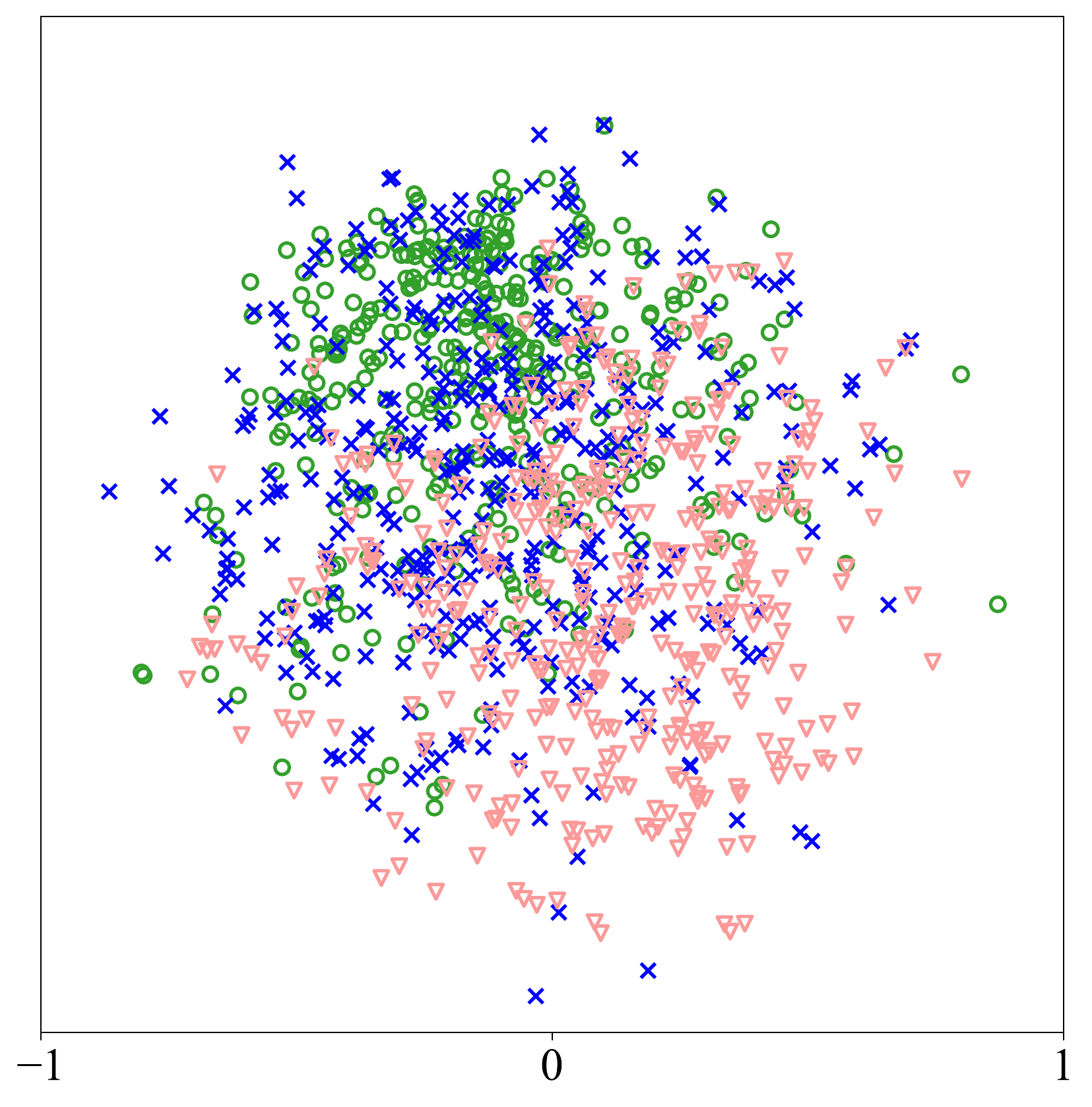} &
\includegraphics[width=0.14\textwidth]{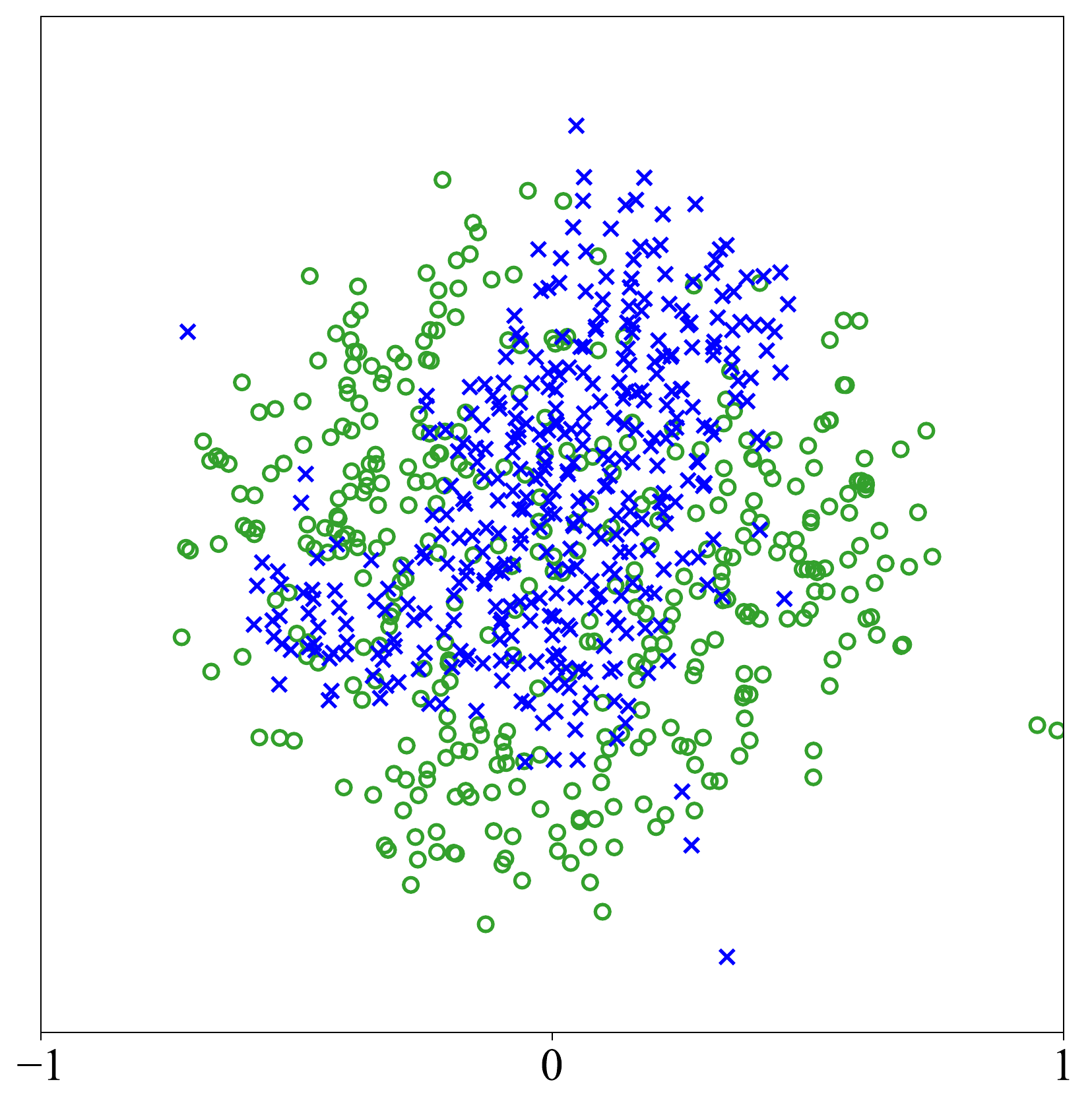}&
\includegraphics[width=0.14\textwidth]{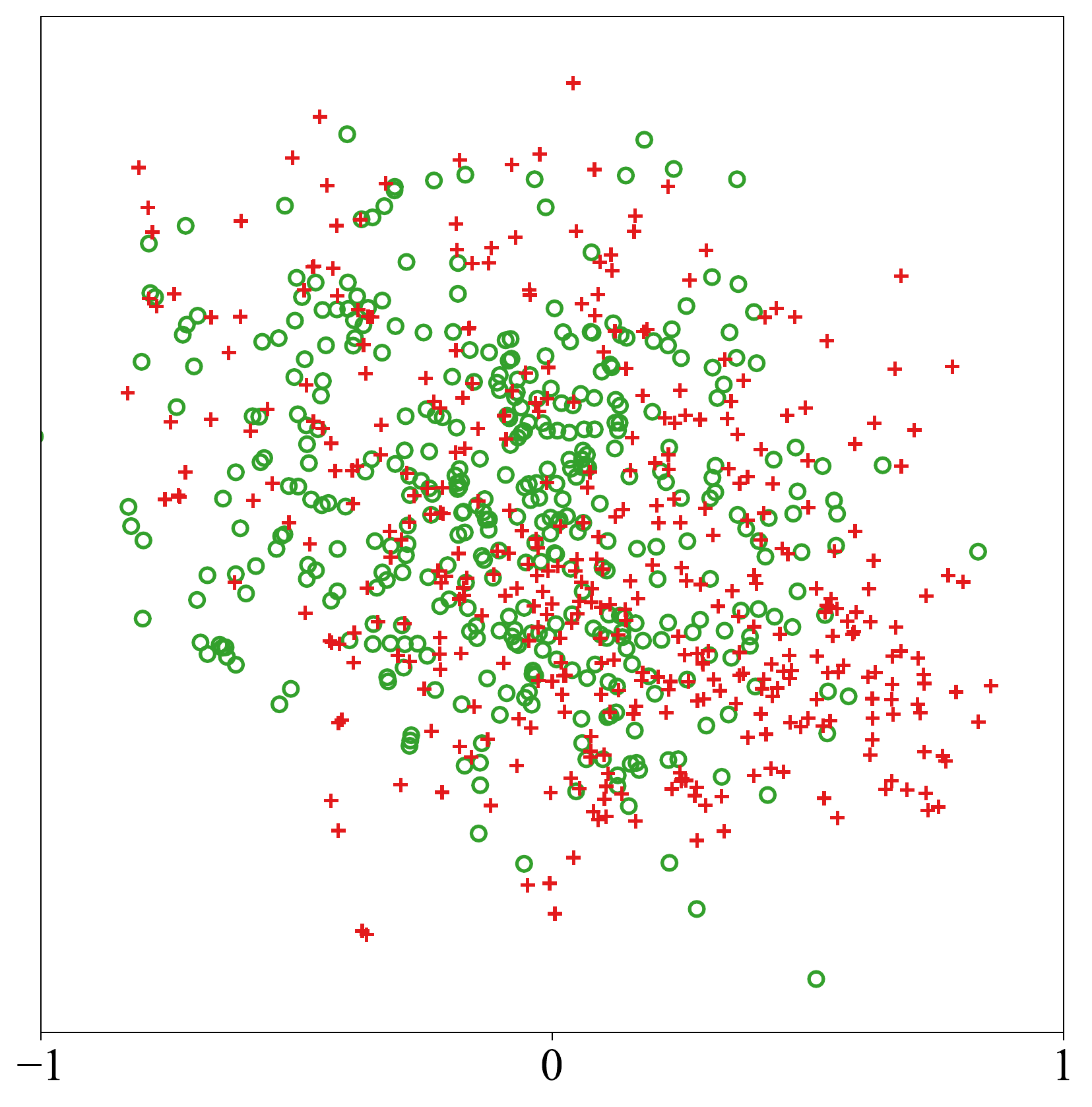} & 
\includegraphics[width=0.14\textwidth]{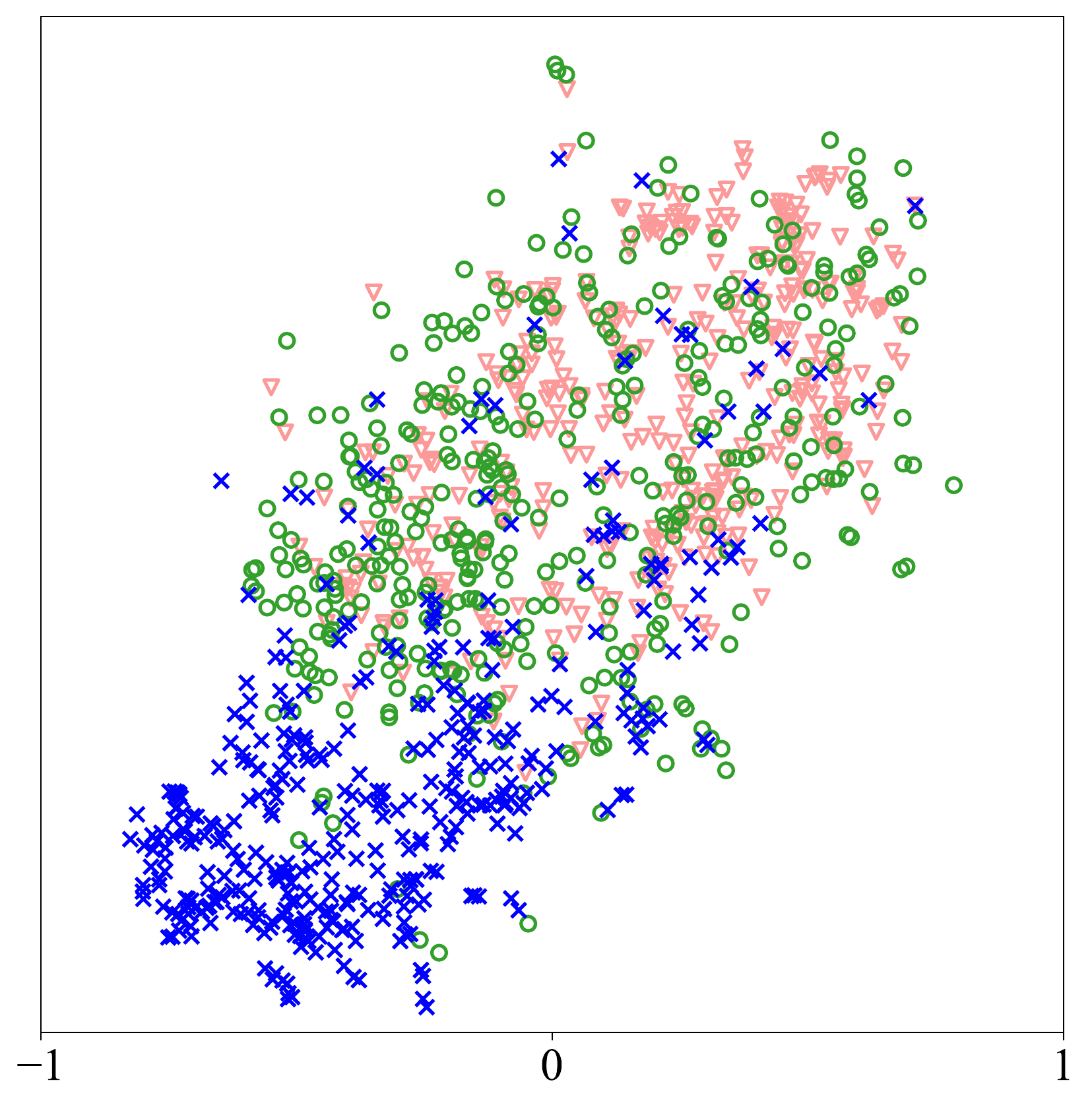}
\\
(a) & (b)  & (c)  & (d) & (e) 
& (f) \\
\multicolumn{6}{c}{\includegraphics[width=0.95\textwidth]{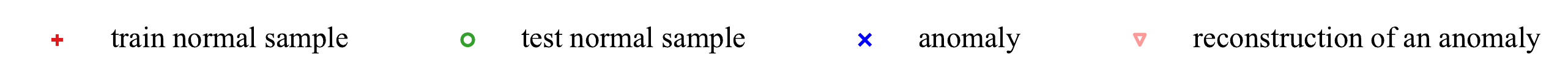}}
\end{tabular}
\caption{Latent analysis for discriminator and generator}
\label{fig:vector-anaylsis}
\end{figure*}

\begin{table}[t]
\caption{Performance comparison on the KDD99 dataset. Ensembles outperform base models and other baselines.}
\label{table:kdd99}
\centering
\begin{tabular}{l|c|c|c}
\hline
Method  & Precision & Recall & F1     \\
\hline
OC-SVM  & 0.746 & 0.8526 & 0.795 \\
DSEBM-r & 0.852 & 0.647 & 0.733\\
DSEBM-e & 0.862 & 0.645 & 0.740\\
DAGMM   & 0.929 & 0.944 & 0.937 \\\hline
EGBAD   & 0.920   & 0.958 & 0.939 \\
f-AnoGAN  & 0.935 & 0.986 & 0.960 \\\hline
$\text{EGBAD}^\textbf{en}$    &   \underline{\textbf{0.972}} &  \underline{0.960}& \underline{0.966} \\
$\text{f-AnoGAN}^\textbf{en}$  & \underline{0.967}   & \underline{\textbf{0.990}} & \underline{\textbf{0.979}} \\\hline
% Ours    & \textbf{0.9681}   & \textbf{0.9873} & \textbf{0.9776}
\end{tabular}
\end{table}
\subsubsection{KDD99 dataset} 
KDD99 is a benchmark dataset for anomaly detection. Each sample in it is a 122-dimensional vector. We compare ensembles of EGBAD and f-AnoGAN against two base models and four other baselines (OC-SVM \citep{scholkopf2001estimating} , DSEBM \citep{zhai2016deep}, DAGMM \citep{zong2018deep}). 
% TODO: add citations
We do not include GANomaly and Skip-GANomaly because they are designed for image data. Following prior work, we use precision, recall, and F1-score to evaluate methods in comparison. The performance values of EGBAD, f-AnoGAN, and other baselines are from their original papers, except that the performance values of OC-SVM are from \citep{zhai2016deep}. 

The results in table \ref{table:kdd99} show that the two ensemble models constructed from  f-AnoGAN  and EGBAD  outperform base models by all evaluation measures. The ensemble of f-AnoGAN outperforms all baseline methods.

\begin{figure}[t]
     \centering
     \includegraphics[width=0.4\textwidth]{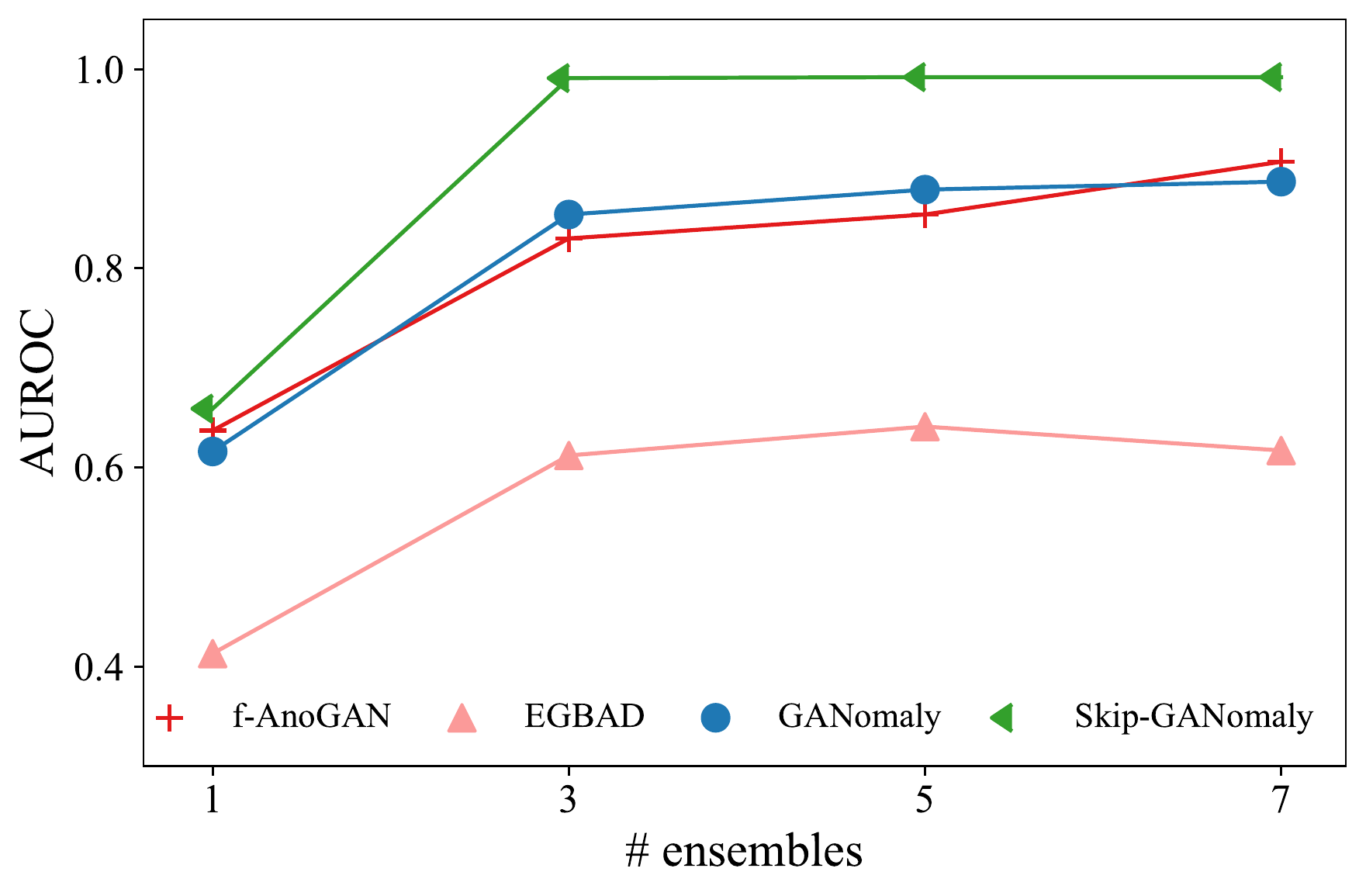}
     \caption{Difference detection performances with different ensemble sizes: $I = J$ $\in$ \{1,3,5,7\}. }
     \label{fig:sensitivity_analysis}
\end{figure}
% \subsubsection{KDD99} 
% To prove that ensemble framework not only can be applied to image data, but also has a good performance on high-dimensional, non-image data. To compare with prior work, we adopt precision, recall and F1-score to evaluate proposed method. \ref{table:kdd99} shows that both f-AnoGAN  and EGBAD can benefit from ensemble and outperform single model.

\subsection{Framework Analysis}
\subsubsection{Analysis of ensemble sizes}
In this experiment, we check how ensemble sizes affect performances. We vary the ensemble sizes in $\{1,3,5,7\}$ and keep the same number of generators and discriminators. We test these configurations on the CIFAR-10 dataset with \textit{ship} as the anomalous class. Figure \ref{fig:sensitivity_analysis} shows that there is a significant improvement from size 1 (a single model) to size 3. The average increase of AUROC is $35.9 \%$ across all models. However, the performance gain from size 3 to size 7 is often marginal. In our experiment with $I = J = 3$, the ensemble takes about 3 times of a base model's running time.. Note that every generator or discriminator in an ensemble is updated once in every 3 iterations. As a rule of thumb, ensembles of size 3 works the best for all tasks in our experiments considering both performance and running time.

\subsubsection{Analysis of encoding vectors and hidden vectors}

We analyze encoding vectors and hidden vectors of normal and abnormal samples to understand why an ensemble significantly improves the performance of a base model. Note that when a sample is present to a GAN, we get an encoding vector from the encoder and a hidden vector from the discriminator's last hidden layer. For an ensemble method, we take the average of encoding vectors and hidden vectors from multiple models. 

We first check encoding vectors. From (a) and (d) in Figure \ref{fig:vector-anaylsis}, we see that encoding vectors of normal samples and abnormal samples are mixed together. This is because the encoder does not encode the unique properties of abnormal samples because the encoder is trained to compress common patterns. Then it means that the reconstruction of abnormal samples will be like normal samples. 

We then check hidden vectors of normal samples from the discriminator. Figure \ref{fig:vector-anaylsis} (b) shows that training and testing samples get different representations in the discriminator in a base model. It means that the discriminator of a single model might overfit the training data. The ensemble seems to train better discriminators by checking (e), which shows that representations of training and test samples are very similar. 

We also check hidden vectors of abnormal samples and their reconstructions. From Figure \ref{fig:vector-anaylsis} (c), we see that the discriminator of a single model cannot distinguish test samples and abnormal samples, then the detection performance is not optimal. Figure \ref{fig:vector-anaylsis} (f) shows that hidden vectors of abnormal samples are far from test normal samples. At the same time, the reconstructions of abnormal samples are similar to normal samples, which is consistent with our observation in (d). Then the difference between an anomaly and its reconstruction is likely to be large in \eqref{eq:loss-d}, then the anomaly will have a large score. 

In summary, the experimental analysis verifies our theoretical analysis in \ref{sect:analysis}. It shows how the ensemble improves the training of discriminators and also the computation of anomaly scores.

We further verify that the ensemble mainly improves the second term of the calculating of anomaly scores in \eqref{eq:anomaly-score}. We vary the value of the relative weight $\beta$ of the reconstruction loss and the discriminative loss and check detection performances. Figure \ref{fig:contribution analysis} shows the results. As $\beta$ increases, the discriminative loss contributes larger fractions to anomaly scores, and the detection performance improves. It means the discriminative loss is important for anomaly detection. Compared with a single model, the ensemble benefits more from large $\beta$ values. It indicates that an ensemble mainly improves the discriminative loss over a base model, as shown in our analysis above.

\begin{figure}[t]
 \centering
 \includegraphics[width=0.4\textwidth]{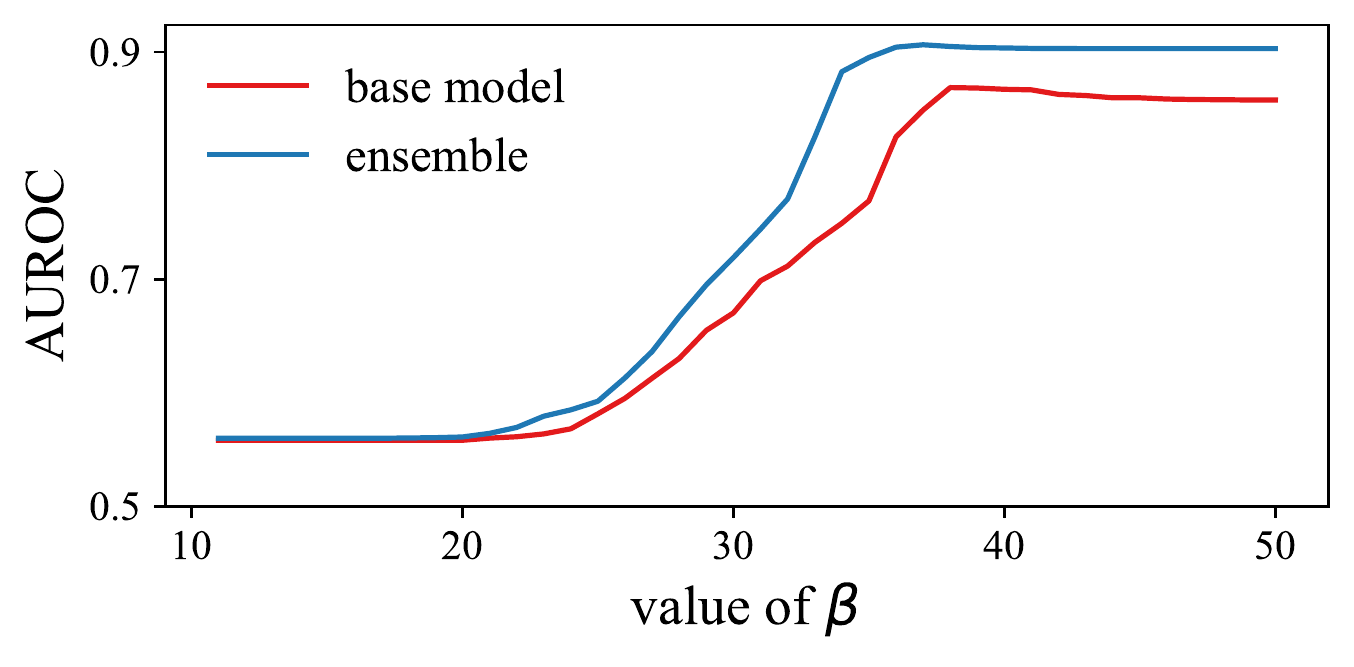}
 \caption{Different detection performances with different relative weight $\beta$ of discriminative loss in \eqref{eq:anomaly-score}.}
 \label{fig:contribution analysis}
\end{figure}

\section{Conclusion}
This work introduces ensemble learning to GAN-based anomaly models for anomaly detection. Extensive experiments show that the ensemble method achieves superior results across various applications over single models. We conduct theoretical analysis to rigorously explain why the discriminator of a GAN is very useful for  anomaly detection and why an ensemble further improves the training of discriminators. Our experimental analysis of encoding vectors and hidden vectors verifies our theoretical analysis and further shows the benefit of ensemble learning.

\section{Acknowledgments}
Thank all reviewers for constructive comments. Li-Ping Liu were supported by DARPA under the SAIL-ON program (Contract number: W911NF-20-2-0006).
Any opinions, findings, and conclusions or recommendations expressed in this material are those of the authors and do not necessarily reflect the views of DARPA.

\bibliography{references.bib}

\newpage

\appendix
\onecolumn 

\begin{table*}[t]
\centering
\scalebox{0.9}{
\begin{tabular}{c|c|c|c|c|c|c|c}
\hline
              & \multicolumn{4}{c}{f-AnoGAN}   & \multicolumn{3}{|c}{GANormaly} \\\hline
Dataset       & MNIST & CIFAR-10 & OCT & KDD99 & MNIST    & CIFAR-10   & OCT  \\\hline
%dimension     & 32*32 &  32*32*3 & 64*64 & 122    & 32*32  &  32*32*3 & 64*64      \\
\# encoding/hidden     & 32    &   32     &  256   &  16    &  128   &   128    &  256   \\
learning rate & 1e-4  & 1e-4     &  1e-4   & 1e-4 &  2e-4  &   2e-4   & 2e-4   \\
GAN type      &  WGAN  &   WGAN  & WGAN & WGAN     &   DCGAN    &DCGAN & DCGAN \\
% optimizer     &  Adam     &          &     &       &          &            &       \\
batchsize     &    64   &     64     &   64  &   128    &     64  &   64& 64\\
epochs        &   25(*3)    &     25(*3)     &  50(*3)   &  30(*3)     &     15(*3)     &20(*3)& 35(*3)\\
$\beta$     & 1  & 9  & 1  &  39 & 9  & 9  &  9 \\\hline
\multicolumn{1}{c}{}\\\hline

    & \multicolumn{4}{c}{EGBAD}   & \multicolumn{3}{|c}{Skip-GAN} \\\hline
Dataset       & MNIST & CIFAR-10 & OCT & KDD99 & MNIST    & CIFAR-10   & OCT  \\\hline
%dimension     & 32*32 &  32*32*3 & 64*64 & 122    & 32*32  &  32*32*3 & 64*64      \\
\# encoding/hidden     & 100    &   100     &  256   &  32    &  100   &   100    &  256   \\
learning rate & 2e-4  & 2e-4     &  2e-4   & 2e-4 &  2e-4  &   2e-4   & 2e-4   \\
GAN type      &  BiGAN  &   BiGAN  & BiGAN & BiGAN     &   UNET    &UNET & UNET \\
% optimizer     &  Adam     &          &     &       &          &            &       \\
batchsize     &    256   &     256     &   64  &   1024    &     256  &   256 & 64\\
epochs        &   10(*3) & 10(*3) &  25(*3) & 5(*3)&     10(*3)&10(*3)& 25(*3)\\
$\beta$     & 0.1  & 0.1  & 9  &  0.1 & 0.1  & 0.1  &  9 \\\hline
\end{tabular}
}
\caption{Hyperparameter details for all learning models on all datasets. These hyperparameter settings are all from previous work obtained their best results. The training of an ensemble model uses three times of the number of training epochs of a base model (see the ``Batch training'' subsection.). ``\# encoding/hidden'' is the dimension of encoding vectors from encoders as well as the dimension of hidden vectors from the discriminator. }
\label{details}
\end{table*}
\section{Appendix}
\subsection{Dataset Details}
\label{appendix:1}
\subsubsection{MNIST dataset} The MNIST dataset  \citep{lecun-mnisthandwrittendigit-2010} contains 10 classes of images of hand-written digits from 0 to 9. In the anomaly detection setting, we create 10 different splits by treating each class as anomalous and the remaining 9 classes as normal. In each split, the training set has 5,4000 normal samples; and the test set contains 7,000 anomalous samples and 9,000 normal samples. 

\subsubsection{CIFAR-10 dataset } 
CIFAR-10 \citep{krizhevsky2009learning} consists of color images with size $32 \times 32$. It also have 10 classes. Similar to previous dataset, we create 10 splits with each class as an anomalous class. In a split,the training set has 45,000 normal samples; and the test set has 9,000 normal samples and 6,000 anomalous samples.

\subsubsection{OCT dataset} 
The Optical Coherence Tomography (OCT) dataset  \citep{oct} contains high-resolution clinical images. It has four classes: CNV, DME, DRUSEN, and NORMAL.  There are 26,315 instances for the NORMAL class, and we treat these as normal instances. The original dataset has a test set with 250 instances for each of the CNV, DME, and DRUSEN classes. We treat these three classes as abnormal classes. We create four splits from the dataset. In the first three splits, the normal class is matched with each abnormal class. The training set has (26,315 - 250) normal instances, and the test set has 250 normal instances and 250 abnormal instances. In the last split, the training set has (26,315 - 750) normal instances, and the test set has all anomalous instances and 750 normal instances. 

Following \citet{f-Anogan}, we sample a random patch ($64\times 64$) from an original image  to compute gradients for training. We divide a test image 16 patches and average their anomaly scores. The approach helps to reduce the model complexity, improve the training stability, and accelerate both training and testing. 

\subsubsection{KDD99 dataset}
The KDDCUP99 10-percent dataset \citep{Dua:2019} contains 20\% data labeled as ``normal'' and 80\% data labeled as ``attack''. Following \citet{DAGMM}, we also treat the ``normal'' data as anomalous class and the ``attack'' data as normal class. Furthermore, we randomly sample 75\% of normal instances for training, and the rest data is used for testing. 

\subsection{Experiment Details}
\label{appendix:2}
We implemented our framework in PyTorch \citep{paszke2017automatic} and ran on a server with a Tesla V100 and CUDA 10.0. Table \ref{details} shows detailed hyperparameter settings. Most of settings are from the original work proposing base models. For all experiments, we use Adam optimizer to train the network with a lambda decay $\beta_1$ = 0.5, and momentums $\beta_2$ = 0.999. 

The original implementation of f-AnoGAN is hard-coded for input size $64 \times 64$. We have modified the implementation so that it also works for image patches of size $32 \times 32$. 
The original EGBAD implementation uses a BiGAN that only works for images with size $28 \times 28$. We re-implement the algorithm so that it works for other image sizes. 

In GANormaly and Skip-GANomaly, we follow the same configuration as the original papers \cite{akcay2018ganomaly, akccay2019skip} in most of the experimental settings.

A baseline model uses settings that give the best results in the original papers. For an ensemble, the settings for its base models are the same as the corresponding baseline. An ensemble is trained with three times of training epochs of a baseline, then a base model in an ensemble gets the same number of updates as its corresponding baseline model.

% \subsection{Inference Time Details}
% Table \ref{table:inference time} shows average inference time for base models and ensemble models. All datasets were run on NVIDIA Tesla V100 with Pytorch 1.4.

% \begin{table}[H]
% \caption{Average inference time over 500 batches (ms)}
% \label{table:inference time}
% \centering
% \begin{tabular}
% {l|cccc}
% \hline
% Method & MNIST & CIFAR-10 & OCT & KDD\\\hline
% f-AnoGAN & &  &  & \\
% EGBAD & &  &  & \\
% GANomaly & &  &  & \\
% Skip-GANomaly & & & & \\\hline
% $\text{f-AnoGAN}^\textbf{en}$ & &  & &\\
% $\text{EGBAD}^\textbf{en}$ & & & &\\
% $\text{GANomaly}^\textbf{en}$ & &  &  &\\
% $\text{Skip-GANomaly}^\textbf{en}$ & &  &  & \\
% \hline
% \end{tabular}
% \end{table}

\end{document}